\documentclass{article}

\usepackage[preprint]{neurips_2026}

\usepackage[utf8]{inputenc} % allow utf-8 input
\usepackage[T1]{fontenc}    % use 8-bit T1 fonts
\usepackage{hyperref}       % hyperlinks
\usepackage{url}            % simple URL typesetting
\usepackage{booktabs}       % professional-quality tables
\usepackage{amsmath}        % display math, bmatrix, lVert/rVert
\usepackage{amsfonts}       % blackboard math symbols
\usepackage{amssymb}        % \lesssim and other symbols
\usepackage{nicefrac}       % compact symbols for 1/2, etc.
\usepackage{microtype}      % microtypography
\usepackage{xcolor}         % colors
\usepackage{graphicx}       % \includegraphics
\usepackage{subcaption}     % subfigures with their own captions
\usepackage{amsthm}         % theorem-like environments
\usepackage{algorithm}      % algorithm float
\usepackage{algpseudocode}  % algorithmic pseudocode

\newtheorem{theorem}{Theorem}[section]
\newtheorem{corollary}[theorem]{Corollary}
\newtheorem{lemma}[theorem]{Lemma}
\newtheorem{definition}[theorem]{Definition}

\newcommand{\mcal}{\mathcal}
\newcommand{\cbrace}[1]{\left\{#1\right\}}
\newcommand{\ldef}{\stackrel{\Delta}{=}}
\newcommand{\bm}[1]{
  \begin{bmatrix}
    #1
  \end{bmatrix}
}
\newcommand{\norm}[1]{\left\lVert #1 \right\rVert}
\newcommand{\real}[1]{\mathbb{R}^{#1}}

\title{Newton's Lantern: A Reinforcement Learning Framework for Finetuning AC Power Flow Warm Start Models}

\author{%
  Shourya Bose\thanks{Equal contribution.},\quad Helgi Hilmarsson\footnotemark[1],\quad Dhruv Suri\footnotemark[1] \\
  Pravah \\
  \texttt{\{shourya.bose, helgi.hilmarsson, dhruv.suri\}@pravah.com}
}

\begin{document}

\maketitle

% Allow page break right after a section/subsection heading
% (default LaTeX glues heading to next paragraph with \clubpenalty=\@M).
\makeatletter
\def\@afterheading{%
  \@nobreakfalse
  \everypar{%
    \if@nobreak
      \@nobreakfalse
      \clubpenalty\@clubpenalty
      \setbox\z@\lastbox
    \else
      \clubpenalty\@clubpenalty
      \everypar{}%
    \fi}}
\makeatother

\begin{abstract}
  Neural warm starts can sharply reduce the number of Newton-Raphson iterations required to solve the AC power flow problem, but existing supervised approaches generalize poorly on heavily loaded instances near voltage collapse. We prove a lower bound on the Newton-Raphson iteration count that depends on the direction of the warm start error rather than on its magnitude, and show as a corollary that the bound becomes vacuous as the smallest singular value of the power-flow Jacobian shrinks, identifying the failure mode of supervised regression near the saddle-node bifurcation. Motivated by this analysis, we introduce \emph{Newton's Lantern}, a finetuning pipeline that combines group relative policy optimization with a learned reward model trained on perturbations of the base model's predictions, using the iteration count itself as the supervisory signal. Across IEEE 118-bus, GOC 500-bus, and GOC 2000-bus benchmarks, Newton's Lantern is the only method that converges on every test snapshot while attaining the smallest mean iteration count.
\end{abstract}

\section{Introduction}
\label{sec:intro}

The electric grid is instrumental in providing a reliable supply of electricity to billions of people worldwide. A critical component of grid operations is the AC power flow (ACPF) problem~\citep{stott1974review}, which determines the steady-state operating point of the system. At its core, ACPF is a root-finding problem for a set of nonlinear equations $g_D(x)$ describing the grid, wherein $x$ represents the system state and $D$ a single snapshot of the grid characteristics. Here, $x$ comprises the bus voltage magnitudes and angles to be determined, while $D$ encapsulates the grid's physical and operational parameters. It is conventionally solved using the Newton-Raphson (NR) algorithm~\citep{tinney1967newton}. An important objective for utility companies and independent system operators is to determine a good starting point $x_0$ for the NR algorithm such that it converges in the fewest iterations. The conventionally accepted solution is the~\emph{flat start}~\citep{tinney1967newton}, which sets all voltage magnitudes to unity and angles to zero, or the~\emph{DC start}, which uses linear DC power flow to generate angles while retaining the unity voltage magnitudes. However, as the size of the grid (and therefore the resulting ACPF problem) grows, both methods can fail to achieve convergence. One category of solutions modifies the NR algorithm itself to improve its convergence characteristics~\citep{iwamoto1981load}, but these methods fail when the region of convergence is far from the flat start point. The second category of solutions, which encompasses the present work, exploits the structure of $g_D(\,\cdot\,)$ or a dataset of solutions corresponding to different instances of $D$ to generate a good $x_0$.

In the latter category, the analytical method of choice has conventionally been \emph{continuation} power flow, which obtains the solution of a `difficult' instance $D'$ by tracing a path homotopy from the solution of a related `easy' instance $D$~\citep{ajjarapu1992continuation, mehta2016numerical}. Recently however,   there has been increased interest in using learning-based methods to approximate the mapping $D \mapsto x_0$ from datasets of pre-solved ACPF instances. While there exists a plethora of work in the supervised setting (see~\citep{khaloie2025review} for a detailed review), applications of reinforcement learning (RL) for this purpose~\citep{yan2025data} have been scarce. This is disappointing, since ACPF solvers produce rich traces during the iterative solution process, which can be used to improve model performance without explicit labels. Another aspect which has been understudied, and directly relates to the necessity of RL for this problem is the contracted region of convergence when the grid is near its \emph{voltage collapse} point~\citep{sauer1990power}. This implies that warm start models trained on grids operating under nominal conditions fail to generalize to the aforementioned cases. 

We specifically focus on three shortcomings with the present literature.

\begin{enumerate}
  \item \textbf{Supervised warm start models use ACPF solutions as targets:} A significant amount of prior literature on supervised warm start models for ACPF~\citep{okhuegbe2024ml, okhuegbe2024deep, diehl2019warmstarting} uses the exact solution $x^*$ from existing datasets as the target for the model. As will be shown in the sequel, this results in poor warm start performance when the unit vector $\frac{\hat{x}^* - x^*}{\norm{\hat{x}^*-x^*}}$ points in certain directions, where $\hat{x}^*$ is the model prediction.
  \item \textbf{Heavily loaded systems near voltage collapse are not considered:} Most ML-for-ACPF datasets are generated by uniformly perturbing a lightly loaded operating point, so the entire dataset stays in the voltage stability regime. The recent PF$\Delta$ benchmark~\citep{rivera2025pfdelta} is a notable exception, using continuation methods to include near collapse cases, but the models trained therein are not used for warm starts. This results in a gap in literature concerning generalization of warm start models.
  \item \textbf{Reinforcement learning has been underutilized:} To our knowledge, only two prior works apply RL to improve ACPF convergence: \citet{yan2025data} modifies each NR step, while \citet{kaseb2025quantum} requires a quantum-enhanced RL environment. Neither yields a plug-and-play warm start model which simply generates a $x_0$ that minimizes the number of NR iterations.
\end{enumerate}

To overcome these shortcomings, we introduce \emph{Newton's Lantern}, an RL framework for finetuning warm start models that are pretrained on datasets restricted to voltage stable operating points away from the voltage collapse point, and are subsequently finetuned on a difficult holdout dataset containing near collapse cases. This setup is representative of practical grid operating conditions, where overloading and voltage collapse occur only under rare adverse events such as extreme weather and demand spikes~\citep{khazeiynasab2021resilience}. Our contributions are concretely as follows.

\begin{enumerate}
  \item \textbf{Characterizing the shortcoming of ACPF targets:} We provide a theoretical result showing that the NR iteration count depends on the direction of the prediction error of warm start models trained with the ACPF solution $x^*$ as the supervised target. A representative 2-bus example is analyzed in Appendix [FUTURE].
  \item \textbf{Considering heavily loaded systems:} The above result implies that near the voltage collapse point, anisotropy of the region of convergence causes the NR iteration count to inflate severely in certain directions. We demonstrate experimentally that warm start models trained on voltage stable cases fail to produce a starting point for which NR converges, while supervised finetuning (SFT) on a holdout set of voltage collapse cases restores convergence.
  \item \textbf{Reinforcement learning for improving convergence:} Our framework \emph{Newton's Lantern} uses group-relative policy optimization (GRPO) with a negative iteration-count reward to finetune pretrained warm start models on the voltage collapse holdout. Experiments show it outperforms SFT on multiple grid examples, as well as a proximal policy optimization (PPO) baseline with an oracle value function.
\end{enumerate}

\paragraph{Notation.}
$\real{n}$ denotes $n$-dimensional Euclidean space and $S^{n-1} \ldef \cbrace{v \in \real{n} : \norm{v} = 1}$ the unit sphere. All norms $\norm{\cdot}$ are Euclidean unless an explicit subscript appears. The class $C^3$ collects functions with continuous third derivatives. For a vector-valued map $g \colon \real{n} \to \real{m}$, the Hessian $\nabla^2 g$ is a tensor of shape $m \times n \times n$, and $\nabla^2 g\,[v, v] \in \real{m}$ is the vector with $\ell$-th entry $\sum_{i,j} v_i v_j\, \partial_i \partial_j\, g_\ell$. The ACPF residual at snapshot $D$ is $g_D \colon \real{2N} \to \real{2N}$, with Jacobian $J_D = \partial g_D / \partial x$ and Hessian $H_D = \nabla^2 g_D$. A solution of $g_D(x) = 0$ is denoted $x^*$, a model's prediction by $\hat x^*$, and the NR warm start by $x_0$. The smallest singular value of a matrix $M$ is $\sigma_{\min}(M)$.

\section{Background}

\subsection{AC Power Flow}
Let $\mcal{G}=\cbrace{\mcal{N},\mcal{E}}$ be a power network with $N \ldef |\mcal{N}|$ buses. The ACPF residual $g_D(x) \ldef \bm{\Delta P^\top & \Delta Q^\top}^\top$ is given by the bus-wise mismatches
\begin{align*}
  \Delta P_i &= P_i^{\text{spec}} - V_i \sum_{j \in \mcal{N}} V_j \left( G_{ij} \cos(\theta_i - \theta_j) + B_{ij} \sin(\theta_i - \theta_j) \right), \\
  \Delta Q_i &= Q_i^{\text{spec}} - V_i \sum_{j \in \mcal{N}} V_j \left( G_{ij} \sin(\theta_i - \theta_j) - B_{ij} \cos(\theta_i - \theta_j) \right),
\end{align*}
where $V_i, \theta_i$ are the voltage magnitude and angle at bus $i$, $P_i^{\text{spec}}, Q_i^{\text{spec}}$ are the specified active and reactive injections, and $G_{ij}, B_{ij}$ are the conductance and susceptance entries of the bus admittance matrix. The grid snapshot $D$ collects these admittances, the specified injections, and a bus-type label (PQ, PV, or slack) per bus, which pins two of the 4-tuple $(V_i, \theta_i, P_i^{\text{spec}}, Q_i^{\text{spec}})$: $(P,Q)$ at PQ buses, $(P,V)$ at PV buses, and $(V,\theta)$ at the slack bus. The free unknowns therefore form an NR system of effective dimension $N_{\text{PV}} + 2 N_{\text{PQ}}$, where $N_{\text{PV}}, N_{\text{PQ}}$ count the respective bus types. Following standard convention in ACPF software and literature on warm start models, we represent the state in the over-parameterized space $x \in \real{2N}$ of all bus voltages and angles; entries pinned by $D$ are overwritten before NR is invoked.

\subsection{NR Method and Voltage Collapse}

The NR method is an iterative root finding algorithm which, starting from an initial point $x_0$, proceeds as
\begin{align*}
  x_{k+1} = x_k - \left[J_D(x_k)\right]^{-1} g_D(x_k),
\end{align*}
and terminates at iteration $k(x_0,\tau) \ldef \inf\cbrace{k \geq 1 : \norm{x_k - x_{k-1}} < \tau}$, where $\tau$ is a predefined termination tolerance. Here, $J_D(x) \ldef \partial g_D(x) / \partial x$ is the Jacobian of the residual, commonly referred to as the \emph{power-flow Jacobian}. \emph{Voltage collapse} refers to the phenomenon wherein, as the load on the grid grows, bus voltages decline progressively until the ACPF loses its solution entirely, marking the steady-state stability limit~\citep{sauer1990power}. Whether a given snapshot $D$ lies near this limit is itself a property of $D$, since the loading, generation dispatch, and topology it encodes determine how close the system operates to collapse. This limit is characterized by the singularity of the power-flow Jacobian at the equilibrium. As $D$ approaches a collapse condition, the smallest singular value $\sigma_{\min}(J_D(x^*))$ shrinks to zero, and is widely used as a proximity index to voltage instability~\citep{tiranuchit1988posturing}. Figure~\ref{fig:collapse} illustrates this on the IEEE 14-bus system. The minimum bus voltage and $\sigma_{\min}(J_D(x^*))$ both collapse to small values as the loading factor $\lambda$ approaches the bifurcation point at which the NR method ceases to converge. The contour in Figure~\ref{fig:collapse-iters} further reveals that the convergence basin around the nominal load is highly \emph{anisotropic}. NR iteration counts grow much faster along certain directions in $(P,Q)$ than others, with the basin pinching off into a narrow tongue near the divergence boundary. This directional sensitivity, which is driven by the singular structure of $J_D$ near collapse, will play a central role in the analysis that follows.

\begin{figure}[t]
  \centering
  \begin{subfigure}[t]{0.32\linewidth}
    \centering
    \includegraphics[width=\linewidth]{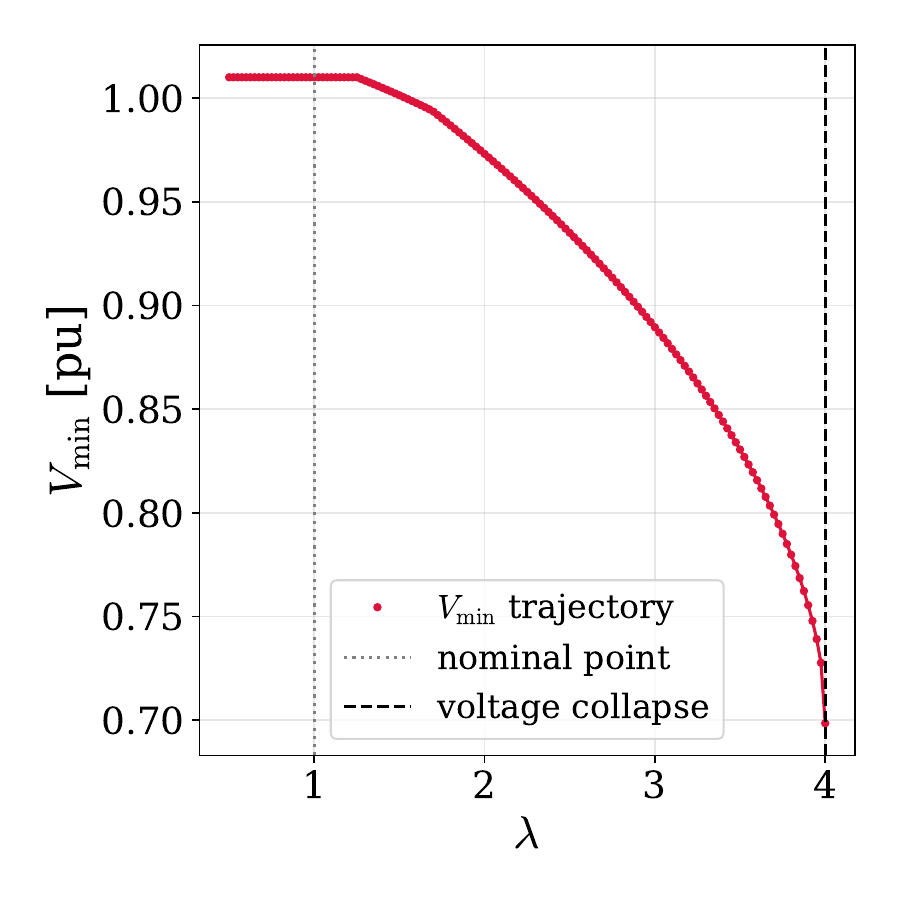}
    \caption{Trajectory of the minimum voltage as a function of load scaling factor $\lambda$.}
    \label{fig:collapse-vmin}
  \end{subfigure}
  \hfill
  \begin{subfigure}[t]{0.32\linewidth}
    \centering
    \includegraphics[width=\linewidth]{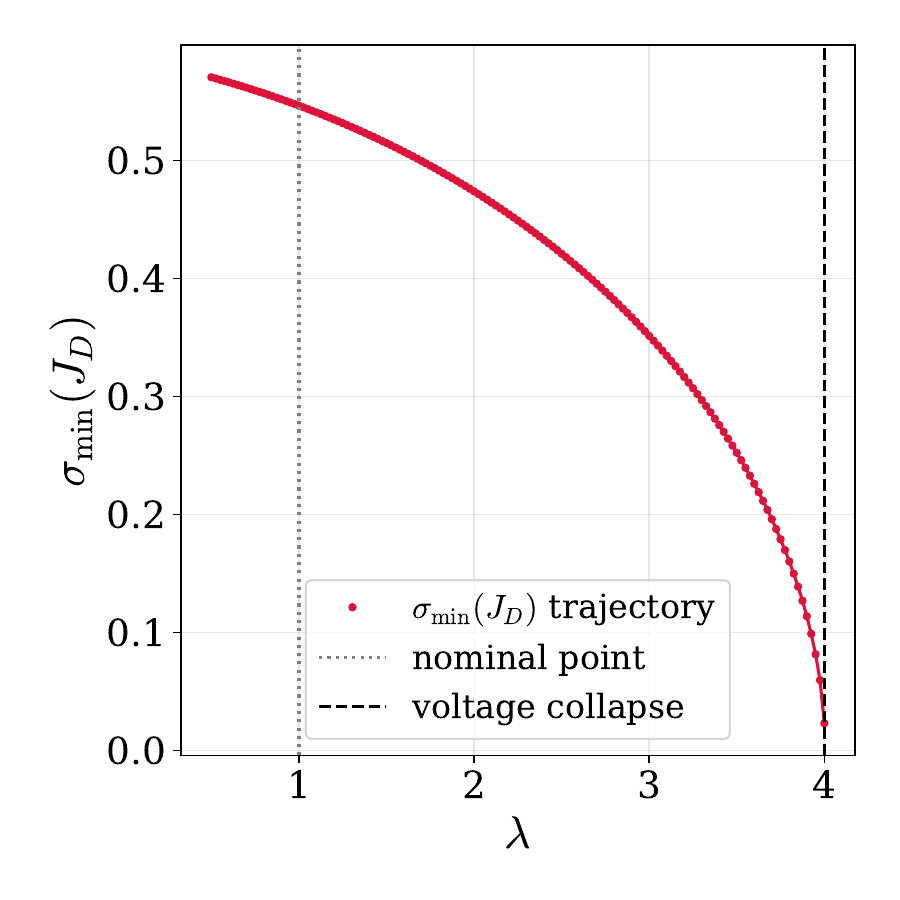}
    \caption{Trajectory of $\sigma_{\min}(J_D(x^*))$ as a function of load scaling factor $\lambda$.}
    \label{fig:collapse-sigma}
  \end{subfigure}
  \hfill
  \begin{subfigure}[t]{0.32\linewidth}
    \centering
    \includegraphics[width=\linewidth]{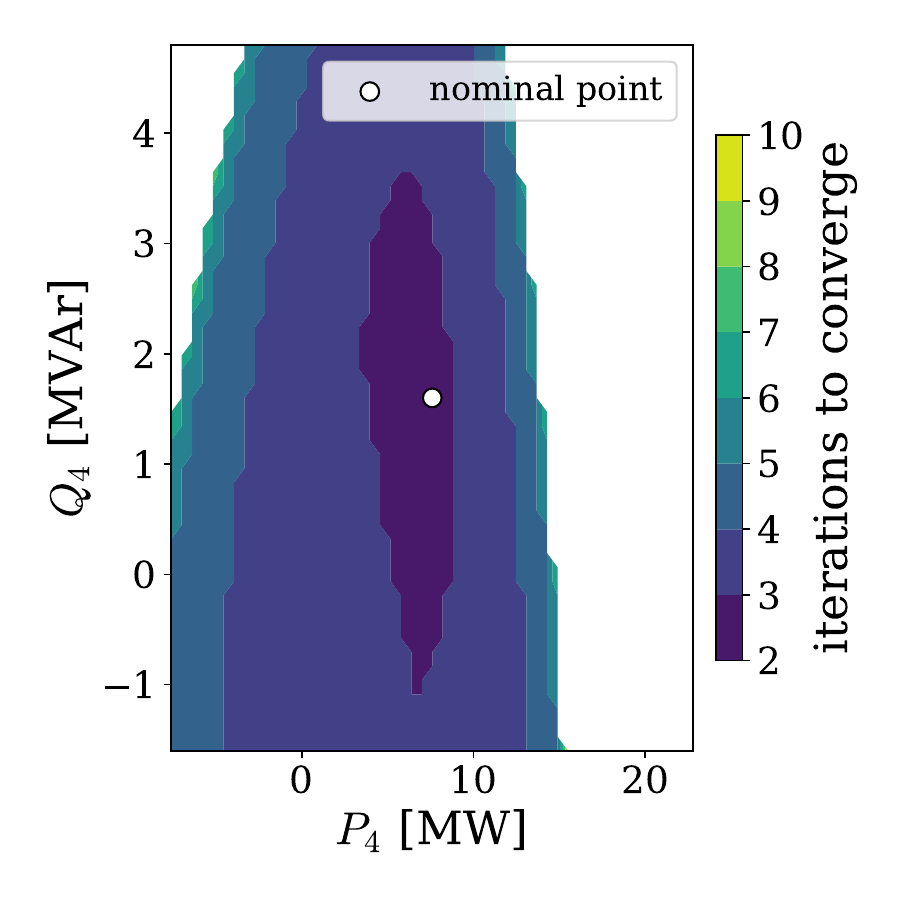}
    \caption{Iteration count for NR convergence over a symmetric grid of $(P,Q)$ load values around the nominal point of the critical bus.}
    \label{fig:collapse-iters}
  \end{subfigure}
  \caption{IEEE 14-bus system: indicators of voltage collapse as the loading factor $\lambda$ scales the nominal load uniformly.}
  \label{fig:collapse}
\end{figure}

\section{Convergence Anisotropy and the Failure of Distance Losses}
\label{sec:analysis}

Supervised warm start models for ACPF are trained against $x^*$ as a regression target~\citep{khaloie2025review, okhuegbe2024ml, diehl2019warmstarting}, under the implicit assumption that a smaller $\norm{\hat x^* - x^*}$ yields fewer NR iterations. This assumption presumes that the convergence region around $x^*$ is isotropic, which Figure~\ref{fig:collapse-iters} and the established literature on load flow fractals~\citep{thorp1989fractals} show is not the case. We quantify the anisotropy through a lower bound on the NR iteration count that depends on the direction of the warm start error rather than its magnitude. We write the warm start error as $e_0 \ldef x_0 - x^*$ in polar form $e_0 = \rho v$ with $\rho \ldef \norm{e_0}$ and $v \ldef e_0/\rho \in S^{2N-1}$. Newton's method requires regularity of $g_D$ and nonsingularity of $J_D(x^*)$~\citep{ortega2000iterative, kelley2003solving, deuflhard2004newton}, which we make precise as the following two assumptions.
\begin{enumerate}
  \item[(A1)] $g_D \in C^3$ on a neighborhood of $x^*$. This holds automatically for ACPF, since $g_D$ is polynomial in voltage magnitudes and trigonometric in angles~\citep{tinney1967newton}.
  \item[(A2)] The Hessian tensor $H_D \ldef \nabla^2 g_D(x^*)$ satisfies $\inf_{v \in S^{2N-1}} \norm{J_D(x^*)^{-1} H_D[v,v]} \geq 2q$ for some $q > 0$. By~\citet{sauer1990power}, this is equivalent to $x^*$ being away from the saddle-node bifurcation, the loading at which $\sigma_{\min}(J_D(x^*)) = 0$.
\end{enumerate}

\begin{theorem}[Lower bound on NR iterations by direction]
\label{thm:lower-bound}
Under (A1) and (A2), there exists $r_0 > 0$, depending only on the third-derivative bound of $g_D$ and on $q$, such that for every warm start $x_0 = x^* + \rho v$ with $\rho \in (\tau, r_0)$,
\begin{align}
  k(x_0, \tau) \;\geq\; \log_2\!\left( \frac{\log(1/\tau)}{\log(1/\rho) - \Lambda(v;D)} \right) - 1,
\end{align}
whenever $\log(1/\rho) - \Lambda(v;D) > 0$. The function $\Lambda \colon S^{2N-1} \to \real{1}$ is a discounted average of $\log\norm{Q_D(\,\cdot\,)}$ along the orbit of $v$ under the Newton-direction map on the sphere, where $Q_D(v) \ldef \frac{1}{2} J_D(x^*)^{-1} H_D[v, v] \in \real{2N}$. The contribution of the $j$-th iterate to $\Lambda$ is weighted by $2^{-j-1}$. Definitions and proof are in Appendix~\ref{app:thm}.
\end{theorem}

The bound is doubly logarithmic in the warm start magnitude $\rho$ and only singly logarithmic in $\Lambda(v;D)$, so direction has the larger leverage. A distance loss $\norm{\hat x^* - x^*}$ carries no representation of $\Lambda$, and two models with identical mean squared error can produce warm starts requiring very different iteration counts.

\begin{corollary}[Voltage collapse uniformly degrades the lower bound]
\label{cor:collapse-degeneracy}
Let $\cbrace{D_\lambda}_{\lambda \in [0, \lambda^*)}$ be a smooth path of grid snapshots approaching a saddle-node bifurcation at $\lambda = \lambda^*$, so $\sigma_\lambda \ldef \sigma_{\min}(J_{D_\lambda}(x_\lambda^*)) \to 0^+$ as $\lambda \to \lambda^*$. Then for $v$ outside a Lebesgue-measure-zero subset of $S^{2N-1}$,
\begin{align}
  \Lambda(v; D_\lambda) \;=\; \log(1/\sigma_\lambda) + \mcal{C}(v) + o(1),
\end{align}
where $\mcal{C}(v) = O(1)$ is bounded uniformly. The denominator of the bound in Theorem~\ref{thm:lower-bound} therefore vanishes once $\sigma_\lambda \lesssim \rho$, at which point the bound becomes vacuous.
\end{corollary}

Near voltage collapse, the $\log(1/\sigma_\lambda)$ term dominates the bounded correction $\mcal{C}(v)$ for almost every direction, and the convergence basin contracts to size $\sim \sigma_\lambda$. A model trained on voltage stable cases produces $\hat x^*$ with $\norm{\hat x^* - x^*} \gg \sigma_\lambda$ near $\lambda^*$, so the bound becomes vacuous and NR diverges. The measure zero exception consists of directions whose Newton-direction orbit stays orthogonal to the voltage collapse mode of $J_{D_\lambda}(x_\lambda^*)$~\citep{gao1992modal, dobson1992observations}, and these correspond to the long axis of the narrow tongue in Figure~\ref{fig:collapse-iters}.

Figure~\ref{fig:analysis} collects the empirical content of this section on the IEEE 14-bus case. Figures~\ref{fig:dir-lambda} and~\ref{fig:dir-iters} test whether the iteration count of NR is sensitive to direction at a fixed warm start magnitude. We pick a moderately loaded operating point at which $\sigma_{\min}(J_D) \approx 0.027$, parameterize a great circle in $S^{2N-1}$ by $v(\theta) = \cos\theta\, w + \sin\theta\, v_\perp$ where $w$ and $v_\perp$ are the two smallest right-singular vectors of $J_D(x^*)$, and sweep $\theta$ at fixed $\rho = 0.05$ and $\tau = 10^{-6}$. The actual iteration count varies with $\theta$ in tight agreement with $\Lambda(v(\theta); D)$, and the lower bound of Theorem~\ref{thm:lower-bound} traces the same shape one step below, confirming that direction is the dominant degree of freedom even when distance is held fixed.

Figure~\ref{fig:lambda-vs-sigma} tests the asymptotic predicted by Corollary~\ref{cor:collapse-degeneracy}. Sweeping the loading factor $\lambda$ along a continuation curve drives $\sigma_{\min}(J_{D_\lambda}(x_\lambda^*))$ to zero, and we measure $\Lambda$ at each $\lambda$ for three direction choices. The three curves track the reference $\log(1/\sigma_\lambda)$ at unit slope with bounded offsets equal to the direction dependent constant $\mcal{C}(v)$ of the corollary, confirming the leading-order shape of the asymptotic. Figure~\ref{fig:bound-validation} validates the predicted bound directly against actual NR iteration counts over $\sim 800$ random triples $(\lambda, v, \rho)$. Every sample respects the bound and most lie within one iteration of it. Color encodes $\sigma_{\min}(J_D)$ at the sampled operating point and confirms that the bound becomes vacuous as $\sigma_{\min} \to 0$.

\begin{figure}[t]
  \centering
  \begin{subfigure}[t]{0.24\linewidth}
    \centering
    \includegraphics[width=\linewidth]{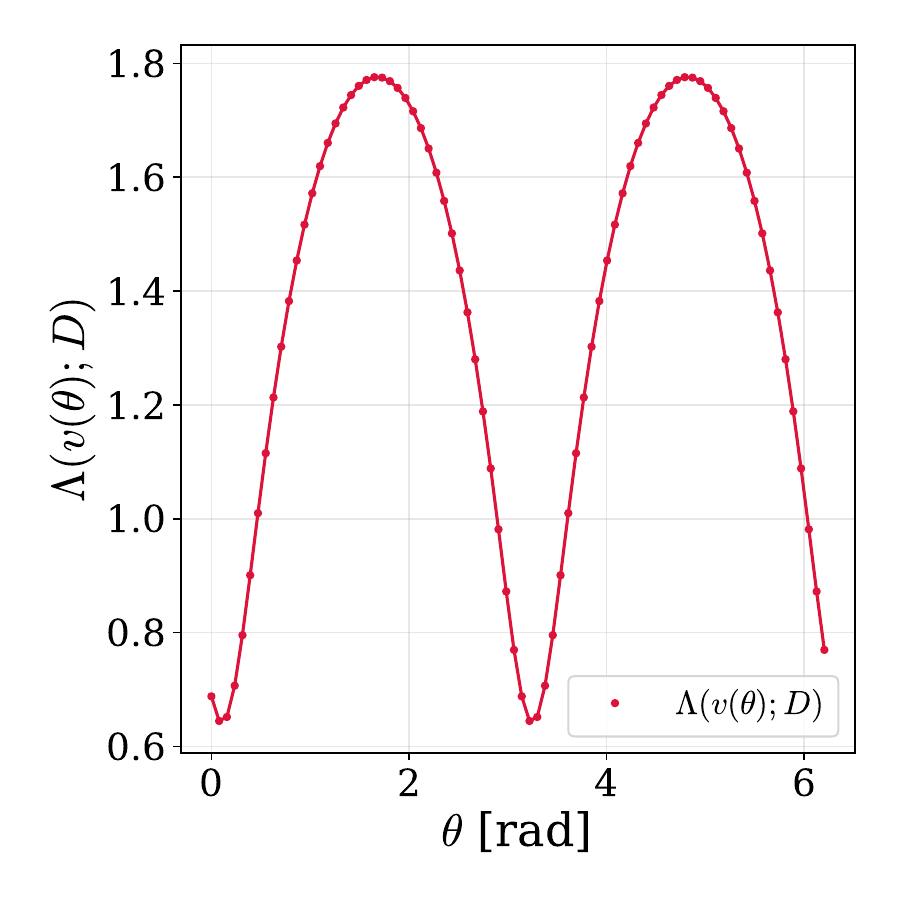}
    \caption{$\Lambda(v(\theta); D)$ along a great circle in $S^{2N-1}$.}
    \label{fig:dir-lambda}
  \end{subfigure}
  \hfill
  \begin{subfigure}[t]{0.24\linewidth}
    \centering
    \includegraphics[width=\linewidth]{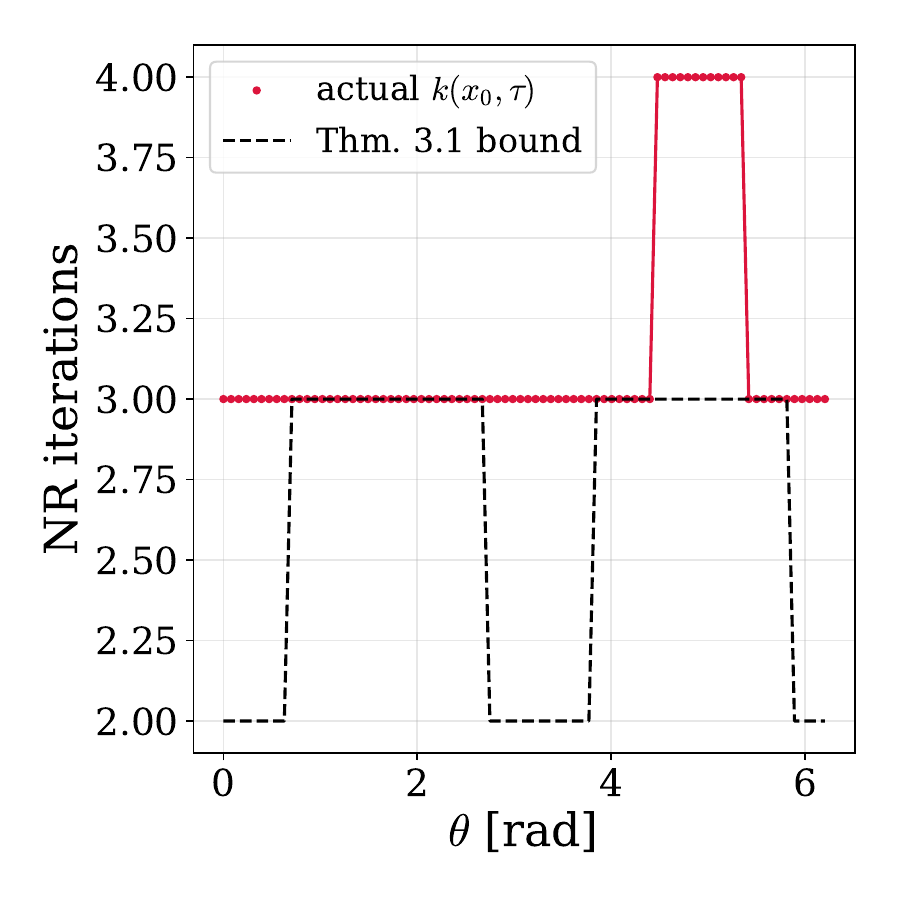}
    \caption{NR iteration count and Theorem~\ref{thm:lower-bound} bound along the same circle.}
    \label{fig:dir-iters}
  \end{subfigure}
  \hfill
  \begin{subfigure}[t]{0.24\linewidth}
    \centering
    \includegraphics[width=\linewidth]{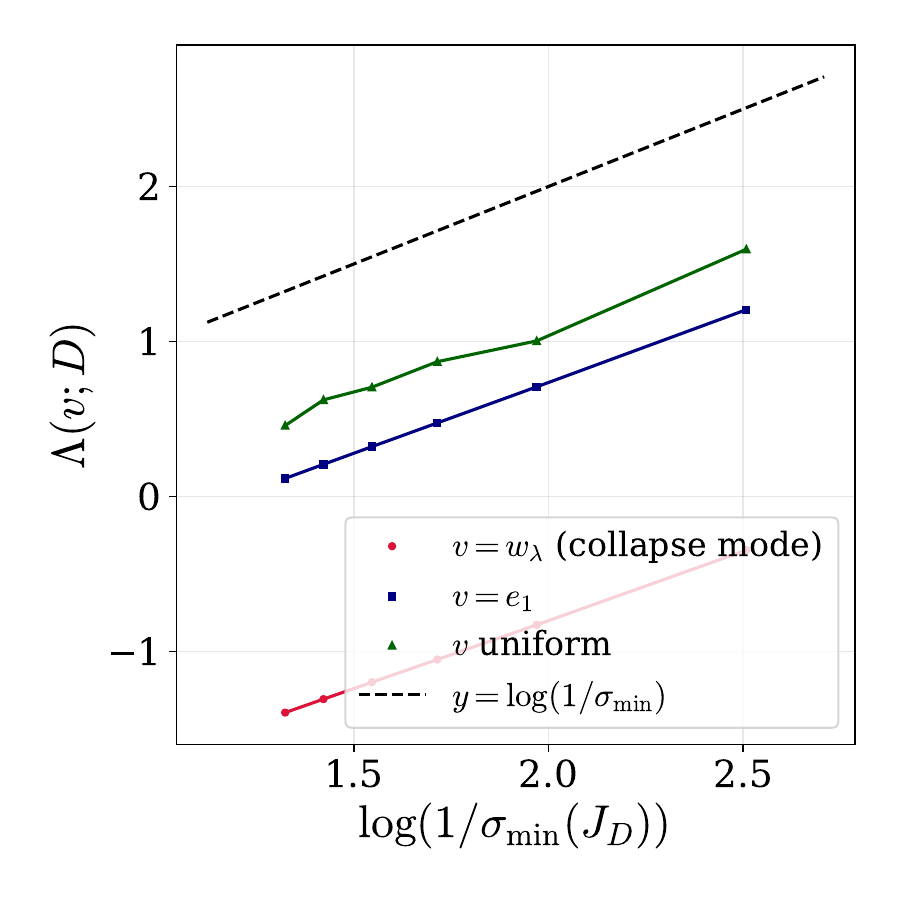}
    \caption{$\Lambda(v; D_\lambda)$ vs.\ $\log(1/\sigma_\lambda)$ for three directions.}
    \label{fig:lambda-vs-sigma}
  \end{subfigure}
  \hfill
  \begin{subfigure}[t]{0.24\linewidth}
    \centering
    \includegraphics[width=\linewidth]{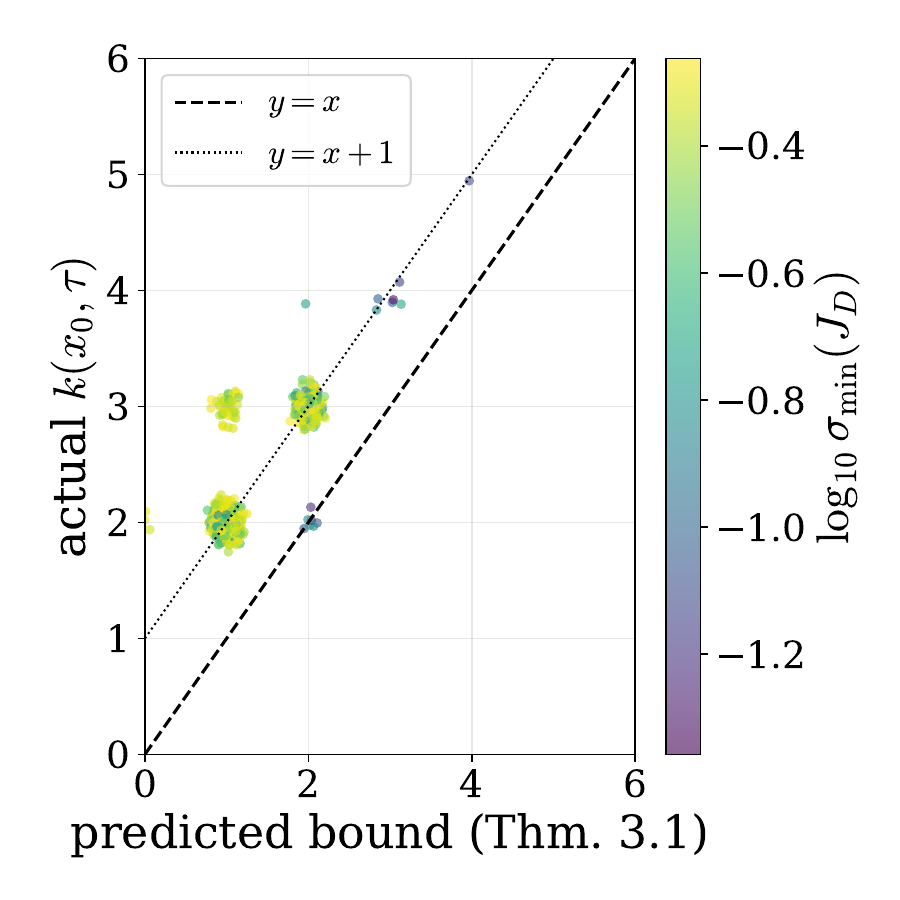}
    \caption{Theorem~\ref{thm:lower-bound} bound vs.\ actual iteration count over random warm starts.}
    \label{fig:bound-validation}
  \end{subfigure}
  \caption{IEEE 14-bus diagnostics for Theorem~\ref{thm:lower-bound} and Corollary~\ref{cor:collapse-degeneracy}.}
  \label{fig:analysis}
\end{figure}

The two results identify the iteration count, not the Euclidean distance to $x^*$, as the quantity a warm start model should optimize. The next section develops a GRPO finetuning pipeline that uses the iteration count directly as a reward signal.

\section{Finetuning with Newton's Lantern}
\label{sec:method}

Consider the setting in which a base warm start model $f_\theta$ has been pretrained on a large supervised dataset of voltage stable operating points, and the modeler is granted a smaller holdout of more difficult instances on which to finetune. The role of the holdout is to expose $f_\theta$ to the near voltage collapse regime identified in Section~\ref{sec:analysis} as the failure mode of distance based losses. The PF$\Delta$ benchmark of \citet{rivera2025pfdelta}, populated by continuation power flow runs that approach the bifurcation, is one such source.

The simplest use of the holdout is supervised finetuning (SFT), in which $f_\theta$ is updated against $x^*$ with the same regression loss used during pretraining. We retain SFT as a baseline. By Theorem~\ref{thm:lower-bound} and Corollary~\ref{cor:collapse-degeneracy}, SFT is structurally limited because the regression loss carries no representation of the directional functional $\Lambda(v;D)$ that controls the iteration count near the bifurcation. To break this degeneracy we move to a finetuning objective that depends on the iteration count itself.

\paragraph{Reinforcement learning formulation.}
We cast finetuning as a single-step Markov decision process. The state is the grid snapshot $s \ldef D$ drawn from the holdout, and the action $a \ldef x_0 \in \real{2N}$ is the warm start that the policy supplies to NR. The mean of the policy is the base model's prediction $f_\theta(s)$, perturbed by a Gaussian whose state-independent log standard deviations $\log\sigma_V$ and $\log\sigma_\theta$ are optimized jointly with $\theta$,
\begin{align}
  a = f_\theta(s) + \sigma \odot \epsilon, \qquad \epsilon \sim \mcal{N}(0, I),
\end{align}
where $\odot$ acts only on the free voltage and angle coordinates of the over parameterized state. The policy is updated with proximal policy optimization~\citep{schulman2017proximal}, whose clipped surrogate objective
\begin{align}
  \mcal{L}^{\mathrm{PPO}}(\theta) = \mathbb{E}\!\left[ \min\!\big(\eta_t(\theta) A_t,\ \mathrm{clip}\bigl(\eta_t(\theta), 1-\varepsilon, 1+\varepsilon\bigr) A_t \big) \right],
\end{align}
employs the importance ratio $\eta_t(\theta) = \pi_\theta(a_t \mid s_t)/\pi_{\theta_{\text{old}}}(a_t \mid s_t)$ relative to the data-generating policy, and an advantage $A_t$ formed by subtracting a baseline from the action's reward. The reward and the baseline are the two design choices that separate our two configurations below.

\paragraph{Configuration 1 (PPO with an oracle baseline).}
Because the holdout ships with labeled solutions, an oracle baseline is available without a learned critic. For each snapshot $s$ we precompute $V^\star(s)$, the iteration count of NR seeded from $x^*(s)$. The reward for a sampled action is a bounded saturating function of the iteration count,
\begin{align}
  r_{\mathrm{sat}}(s, a) = \begin{cases}
    R_+ - g\bigl(k(a, \tau)\bigr), & \text{NR converged from } a, \\
    -R_-, & \text{otherwise,}
  \end{cases}
  \qquad g(n) = \frac{n - 1}{n - 1 + c},
  \label{eq:reward-sat}
\end{align}
where $c$ is set to the holdout-mean of $V^\star$ so that $g$ reaches its half-saturation value at the typical iteration count and $R_\pm$ keep the reward bounded. The advantage takes the form
\begin{align}
  A(s, a) = r_{\mathrm{sat}}(s, a) - r_{\mathrm{sat}}\bigl(s, x^*(s)\bigr),
\end{align}
which subtracts the per-snapshot oracle from the rollout reward. We denote this configuration \textbf{PPO+$V^\star$}.

\paragraph{Configuration 2 (GRPO with a learned reward).}
PPO+$V^\star$ inherits two computational costs. Each rollout requires running NR to obtain $r_{\mathrm{sat}}$, and each holdout snapshot incurs an additional Newton solve to compute $V^\star(s)$. We address both with a pair of substitutions, replacing the per-snapshot oracle with a per-batch group baseline and replacing NR with a learned reward model.

For the baseline, we adopt group relative policy optimization~\citep{shao2024deepseekmath}. For each state in the batch we draw $K$ rollouts from the current policy, evaluate their rewards, and standardize within the resulting group,
\begin{align}
  A_k = \frac{r_k - \mu_g}{\sigma_g + \varepsilon_g},
\end{align}
with $\mu_g$ and $\sigma_g$ the within-group mean and standard deviation. The group baseline absorbs the state-dependent offset that $V^\star$ would have removed, without consulting the labeled solution.

For the reward, we replace the NR call with a learned reward model $R_\phi$ that maps the snapshot features and the warm start to a predicted iteration count. The model is trained on perturbations of the base model's own predictions. For each holdout snapshot, we sample a grid of magnitudes and random directions in the free coordinate subspace, run NR once per perturbed warm start, and aggregate the resulting $(s, a, k(a,\tau))$ triples into a regression dataset. Training $R_\phi$ on this distribution calibrates the reward over the region the policy will subsequently visit.

The reward function used in this configuration is not the saturating $r_{\mathrm{sat}}$ of Configuration~1. Rather, we adopt a piecewise linear reward with a convergence bonus,
\begin{align}
  r_{\mathrm{lin}}(s, a) = -R_\phi(s, a) + B \cdot \mathbf{1}\bigl[R_\phi(s, a) < k_{\max}\bigr],
  \label{eq:reward-lin}
\end{align}
where $k_{\max}$ is the iteration cap that defines convergence and $B$ is a constant on the order of $k_{\max}$. The choice was hand-tuned upon two empirical observations made during training. First, $r_{\mathrm{sat}}$ is nearly flat over the iteration counts that GRPO rollouts typically produce on a holdout, and the within-group standardization $A_k$ collapses to noise under such a flat reward; the linear form spreads the reward over the relevant range so that the standardized advantage is informative. Second, without an additional discrete signal at the convergence boundary, GRPO has no way to distinguish a marginally non-converging rollout from a comfortably converging one; the bonus $B$ creates the necessary gap. We denote this configuration \textbf{GRPO+$R_\phi$}.

\paragraph{Newton's Lantern.}
Newton's Lantern is the GRPO+$R_\phi$ configuration together with one additional ingredient. Every $T_v$ outer iterations we evaluate the current policy on a held-out validation slice using real NR, and at the end of training we return the snapshot whose validation iteration count is smallest, rather than the policy of the final outer iteration. This decouples the returned policy from the noise of any single iteration. The full procedure is summarized in Algorithm~\ref{alg:lantern}.

\begin{algorithm}[t]
\caption{Newton's Lantern}
\label{alg:lantern}
\begin{algorithmic}[1]
\Require Pretrained warm start model $f_\theta$; holdout $\mcal{D}_{\text{hold}}$ split into train, val, and test
\Statex \textit{Reward model training:}
\For{$s_i \in \mcal{D}_{\text{train}}$}
  \State $\hat x^*_i \gets f_\theta(s_i)$
  \For{magnitudes $f \in \cbrace{f_1, \ldots, f_M}$ and directions $u_1, \ldots, u_{K_d} \in S^{2N-1}$}
    \State $a \gets \hat x^*_i + f \cdot u$
    \State Run NR from $a$ on snapshot $s_i$, record iteration count $k(a, \tau)$
  \EndFor
\EndFor
\State Train $R_\phi$ on the aggregated $\bigl(s_i, a, k(a, \tau)\bigr)$ pairs.
\Statex \textit{GRPO finetuning:}
\State Wrap $f_\theta$ in a stochastic policy $\pi_\theta$ with learnable $\sigma$.
\For{outer iter $t = 1, \ldots, T$}
  \State Draw a batch $\cbrace{s_b}_{b=1}^B \subset \mcal{D}_{\text{train}}$.
  \For{$b = 1, \ldots, B$}
    \State Sample $K$ rollouts $\cbrace{a_b^k}_{k=1}^K$ from $\pi_\theta(\cdot \mid s_b)$.
    \State Compute rewards $r_b^k \gets r_{\mathrm{lin}}(s_b, a_b^k)$ from Eq.~\eqref{eq:reward-lin}, with $R_\phi$ as the iteration-count predictor.
    \State Compute advantages $A_b^k \gets (r_b^k - \mu_b)/(\sigma_b + \varepsilon_g)$.
  \EndFor
  \State Update $\theta$ via PPO-clipped surrogate over $\cbrace{(s_b, a_b^k, A_b^k)}$.
  \If{$t \bmod T_v = 0$}
    \State Evaluate $\pi_\theta$ on $\mcal{D}_{\text{val}}$ with real NR, snapshot $\theta$ if mean iter count improved.
  \EndIf
\EndFor
\State \Return best snapshot of $\theta$.
\end{algorithmic}
\end{algorithm}

\subsection{Models}

The pipeline above does not depend on the architecture of the base model $f_\theta$ beyond the requirement that $f_\theta(s)$ produce a usable prediction of the ACPF solution. We instantiate Newton's Lantern with two such backbones to span the practical range of warm start architectures in the literature.

The first backbone is a heterogeneous graph neural network (GNN) in the message passing style of \citet{piloto2024canos}, the architecture used by their CANOS solver for AC optimal power flow. Each bus is a node with features encoding demand, generation, and shunt, each line is an edge with features encoding line impedance, susceptance, transformer ratio, and rate limits, and the graph is heterogeneous in that PQ, PV, and slack buses carry distinct node embeddings. Per-node-type encoder multilayer perceptrons (MLPs) lift these inputs to a hidden representation $h^{(0)}_v$ of dimension $d$. The network then performs $T$ rounds of edge updates and node aggregations,
\begin{align}
  m_e^{(t+1)} &= \phi_{\mathrm{type}(e)}^{\,e}\bigl(h_{u(e)}^{(t)},\, h_{v(e)}^{(t)},\, e_{\mathrm{attr}}\bigr), \\
  h_v^{(t+1)} &= \phi_{\mathrm{type}(v)}^{\,n}\!\left(h_v^{(t)},\, \frac{1}{|\mcal{E}(v)|}\!\!\sum_{e \in \mcal{E}(v)} m_e^{(t+1)}\right),
\end{align}
where $\phi^e$ and $\phi^n$ are MLPs whose parameters are shared across edges or nodes of the same type, and $\mcal{E}(v)$ denotes the edges incident to $v$. After the final round, per-node-type decoder MLPs produce $(\hat \theta_i, \hat V_i)$ at PQ buses, $(\hat Q_i, \hat \theta_i)$ at PV buses, and the slack injections $(\hat P_{\text{slack}}, \hat Q_{\text{slack}})$. The components of the over-parameterized state that are pinned by $D$ (slack $V$ and $\theta$, PV $V$) are then overwritten with their setpoints before $\hat x^*$ is returned, so the same model can be queried for warm starts directly. To verify that any improvement we report is attributable to the GRPO pipeline rather than to the inductive bias of the GNN, we also instantiate $f_\theta$ as a fully connected feedforward network (FCNN) on the smallest test grid.

\section{Experiments}
\label{sec:experiments}

We evaluate Newton's Lantern on three grids of increasing size, namely the IEEE 118-bus system (denoted \texttt{case118}), the GOC 500-bus system (\texttt{case500}), and the GOC 2000-bus system (\texttt{case2000}). All snapshots are drawn from the PF$\Delta$ benchmark of \citet{rivera2025pfdelta}\footnote{Hugging Face repository: \url{https://huggingface.co/datasets/pfdelta/pfdelta}.}, which already segregates samples into a large pool of voltage stable instances and a smaller pool of near voltage collapse instances generated via continuation power flow. We use the voltage stable pool for pretraining, and a disjoint subset of the voltage collapse pool for the holdout (which is itself partitioned into train, val, and 30 reserved test snapshots).

The base model $f_\theta$ is the heterogeneous GNN of Section~\ref{sec:method} for \texttt{case500} and \texttt{case2000}, with hidden dimension $d = 128$ and $T = 32$ message passing rounds for \texttt{case500} and $T = 60$ for \texttt{case2000}. For \texttt{case118} we use the FCNN with four hidden layers of width 512 in place of the GNN. All base models are pretrained on the voltage stable pool with the power balance loss (PBL) of \citet{piloto2024canos},
\begin{align}
  \mcal{L}_{\mathrm{PBL}}(\hat x) = \mathbb{E}_{i \in \mcal{N}}\!\left[\sqrt{\Delta P_i(\hat x)^2 + \Delta Q_i(\hat x)^2 + \zeta}\right],
  \label{eq:pbl}
\end{align}
where $\Delta P_i, \Delta Q_i$ are the bus-wise mismatches at the predicted state $\hat x$ and $\zeta = 10^{-12}$ keeps the gradient finite at the root. The pretrained model is then finetuned via one of supervised regression on the holdout (SFT), PPO with the oracle baseline (PPO+$V^\star$), or Newton's Lantern (GRPO+$R_\phi$). All training, RL, and reward model hyperparameters are deferred to Appendix~\ref{app:repro}; in the body we only flag the two values that are not standard. The policy learning rate on the GNN backbones is $10^{-6}$ and the PPO clip radius is $\varepsilon = 0.1$, both an order of magnitude tighter than commonly reported PPO defaults, because larger values cause the GNN policy to collapse within a few outer iterations.

Tables~\ref{tab:case500}, \ref{tab:case2000}, and \ref{tab:case118} report results on 30 randomly drawn test snapshots per grid, evaluated with PowerModels.jl Newton-Raphson at iteration cap 1000. We report convergence rate, mean iteration count restricted to the converged subset, mean iteration count over all samples (with non-converged samples set to the iteration cap), and the mean Euclidean distance and PBL of the warm start to the labeled solution.

\begin{table}[t]
  \centering
  \small
  \caption{\texttt{case500} (GNN, $d=128$, $T=32$, 30 test samples).}
  \label{tab:case500}
  \begin{tabular}{lcrrrr}
    \toprule
    Method & Solved & Iters (solved) & Iters (all) & $\norm{x_0 - x^*}$ & PBL@$x_0$ \\
    \midrule
    Flat start                       & $30/30$ & $15.20$ &  $15.20$ & $14.40$ & $7.59$ \\
    DC start                         & $30/30$ & $15.77$ &  $15.77$ &  $\mathbf{4.32}$ & $7.21$ \\
    SFT                              & $27/30$ & $15.00$ & $113.50$ &  $9.62$ & $\mathbf{1.38}$ \\
    PPO+$V^\star$                    & $29/30$ & $15.17$ &  $48.00$ &  $9.58$ & $1.49$ \\
    Newton's Lantern (GRPO+$R_\phi$) & $\mathbf{30/30}$ & $\mathbf{15.00}$ & $\mathbf{15.00}$ &  $9.66$ & $1.47$ \\
    \bottomrule
  \end{tabular}
\end{table}

\begin{table}[t]
  \centering
  \small
  \caption{\texttt{case2000} (GNN, $d=128$, $T=60$, 30 test samples).}
  \label{tab:case2000}
  \begin{tabular}{lcrrrr}
    \toprule
    Method & Solved & Iters (solved) & Iters (all) & $\norm{x_0 - x^*}$ & PBL@$x_0$ \\
    \midrule
    Flat start                       & $0/30$ & --- & $1000.00$ & $78.43$ & $5.06$ \\
    DC start                         & $0/30$ & --- & $1000.00$ & $\mathbf{38.31}$ & $4.97$ \\
    SFT                              & $30/30$ & $14.87$ & $14.87$ & $55.17$ & $\mathbf{0.99}$ \\
    PPO+$V^\star$                    & $30/30$ & $14.67$ & $14.67$ & $55.19$ & $1.04$ \\
    Newton's Lantern (GRPO+$R_\phi$) & $\mathbf{30/30}$ & $\mathbf{14.03}$ & $\mathbf{14.03}$ & $55.32$ & $1.36$ \\
    \bottomrule
  \end{tabular}
\end{table}

\begin{table}[t]
  \centering
  \small
  \caption{\texttt{case118} (FCNN, 4 hidden layers, width 512, 30 test samples).}
  \label{tab:case118}
  \begin{tabular}{lcrrrr}
    \toprule
    Method & Solved & Iters (solved) & Iters (all) & $\norm{x_0 - x^*}$ & PBL@$x_0$ \\
    \midrule
    Flat start                       & $30/30$ & $18.50$ & $18.50$ & $19.93$ & $3.24$ \\
    DC start                         & $30/30$ & $17.97$ & $17.97$ &  $8.06$ & $2.15$ \\
    SFT                              & $30/30$ & $17.43$ & $17.43$ &  $\mathbf{5.37}$ & $\mathbf{0.66}$ \\
    PPO+$V^\star$                    & $30/30$ & $17.53$ & $17.53$ &  $5.42$ & $0.79$ \\
    Newton's Lantern (GRPO+$R_\phi$) & $\mathbf{30/30}$ & $\mathbf{17.43}$ & $\mathbf{17.43}$ &  $\mathbf{5.37}$ & $\mathbf{0.66}$ \\
    \bottomrule
  \end{tabular}
\end{table}

The three tables exhibit the same qualitative ordering. Newton's Lantern is the only method that converges on all 30 test snapshots across every grid, and it attains the smallest mean iteration count among methods that converge. PPO+$V^\star$ is a fair baseline that recovers some of the divergences of plain SFT, but never matches the GRPO configuration on either convergence rate or mean iteration count. The win margin shrinks as the grid grows. On \texttt{case118}, Newton's Lantern produces the same per-sample iteration counts as SFT. On \texttt{case500}, the iteration counts on the converged subset are tied with SFT, and the gain comes from recovering the three samples that SFT diverges on. On \texttt{case2000}, both SFT and Newton's Lantern converge on all 30 samples, and Newton's Lantern saves $0.84$ iterations on average.

The narrow margins on the two GNN cases are consistent with reports of training instabilities encountered when finetuning GNNs with RL~\citep{dejongh2023ppo}. The unusually small policy learning rate and PPO clip radius noted above were chosen out of necessity, because relaxing either causes the GNN policy to collapse within a few outer iterations. The MLP backbone on \texttt{case118} tolerates standard PPO hyperparameters without instability, which suggests that the limiting factor is the depth and message passing dynamics of the GNN backbone rather than the RL algorithm.

The \texttt{case2000} row offers direct empirical support for Theorem~\ref{thm:lower-bound}. Among methods that converge, Newton's Lantern produces the warm start that is \emph{furthest} from the labeled solution by Euclidean distance ($55.32$ versus $55.17$ for SFT) and that has the largest PBL among the model-based methods ($1.36$ versus $0.99$ for SFT), yet it converges in the \emph{fewest} iterations. The DC start has the smallest distance of any row in the table ($38.31$), but fails to converge on a single sample. These rows together confirm that distance and residual magnitude are inadequate proxies for the iteration count, exactly as predicted by the directional bound of Theorem~\ref{thm:lower-bound}.

\section{Conclusion}
\label{sec:conclusion}

We derived a lower bound on the Newton-Raphson iteration count for AC power flow warm starts that depends on the direction of the warm start error rather than on its magnitude. The directional analysis explained the empirical failure of supervised regression near voltage collapse, and motivated Newton's Lantern, a GRPO finetuning pipeline with a learned reward model that uses iteration count as the supervisory signal. Across three grid sizes, Newton's Lantern was the only method to converge on every test snapshot, while attaining the smallest mean iteration count. The narrow margins observed on the two GNN cases reflect the broader difficulty of training GNN policies with RL. Designing architectures and RL algorithms that handle this regime more gracefully is a natural direction for extending Newton's Lantern to larger grids.

% ===== end boilerplate =====

\newpage
{
\small
\bibliographystyle{plainnat}
\bibliography{refs}

\begin{thebibliography}{27}
\providecommand{\natexlab}[1]{#1}
\providecommand{\url}[1]{\texttt{#1}}
\expandafter\ifx\csname urlstyle\endcsname\relax
  \providecommand{\doi}[1]{doi: #1}\else
  \providecommand{\doi}{doi: \begingroup \urlstyle{rm}\Url}\fi

\bibitem[Ajjarapu and Christy(1992)]{ajjarapu1992continuation}
Venkataramana Ajjarapu and Colin Christy.
\newblock The continuation power flow: A tool for steady state voltage
  stability analysis.
\newblock \emph{IEEE Transactions on Power Systems}, 7\penalty0 (1):\penalty0
  416--423, 1992.

\bibitem[de~Jongh et~al.(2023)de~Jongh, Mueller, Suriyah, and
  Leibfried]{dejongh2023ppo}
Steven de~Jongh, Frederik Mueller, Michael Suriyah, and Thomas Leibfried.
\newblock Proximal policy optimization with graph neural networks for optimal
  power flow.
\newblock In \emph{12th International Conference on Data Science, Technology
  and Applications (DATA)}, 2023.
\newblock arXiv:2212.12470.

\bibitem[Deuflhard(2004)]{deuflhard2004newton}
Peter Deuflhard.
\newblock \emph{{N}ewton Methods for Nonlinear Problems: Affine Invariance and
  Adaptive Algorithms}.
\newblock Springer, 2004.

\bibitem[Diehl(2019)]{diehl2019warmstarting}
Florian Diehl.
\newblock Warm-starting {AC} optimal power flow with graph neural networks.
\newblock In \emph{NeurIPS Workshop on Tackling Climate Change with Machine
  Learning}, 2019.

\bibitem[Dobson(1992)]{dobson1992observations}
Ian Dobson.
\newblock Observations on the geometry of saddle node bifurcation and voltage
  collapse in electrical power systems.
\newblock \emph{IEEE Transactions on Circuits and Systems I: Fundamental Theory
  and Applications}, 39\penalty0 (3):\penalty0 240--243, 1992.

\bibitem[Gao et~al.(1992)Gao, Morison, and Kundur]{gao1992modal}
Bei Gao, Graham~K. Morison, and Prabha Kundur.
\newblock Voltage stability evaluation using modal analysis.
\newblock \emph{IEEE Transactions on Power Systems}, 7\penalty0 (4):\penalty0
  1529--1542, 1992.

\bibitem[Hubbard et~al.(2001)Hubbard, Schleicher, and
  Sutherland]{hubbard2001roots}
John Hubbard, Dierk Schleicher, and Scott Sutherland.
\newblock How to find all roots of complex polynomials by {N}ewton's method.
\newblock \emph{Inventiones Mathematicae}, 146:\penalty0 1--33, 2001.

\bibitem[Iwamoto and Tamura(1981)]{iwamoto1981load}
Shinichi Iwamoto and Yasuo Tamura.
\newblock A load flow calculation method for ill-conditioned power systems.
\newblock \emph{IEEE Transactions on Power Apparatus and Systems},
  PAS-100\penalty0 (4):\penalty0 1736--1743, 1981.

\bibitem[Kaseb et~al.(2025)]{kaseb2025quantum}
Zeynab Kaseb et~al.
\newblock Quantum-enhanced reinforcement learning for accelerating
  {N}ewton-{R}aphson convergence with {I}sing machines: A case study for power
  flow analysis.
\newblock \emph{arXiv preprint arXiv:2511.20237}, 2025.

\bibitem[Kelley(2003)]{kelley2003solving}
C.~T. Kelley.
\newblock \emph{Solving Nonlinear Equations with {N}ewton's Method}.
\newblock SIAM, 2003.

\bibitem[Khaloie et~al.(2025)Khaloie, Dolanyi, Toubeau, and
  Vall\'ee]{khaloie2025review}
Hooman Khaloie, Mihaly Dolanyi, Jean-Fran\c{c}ois Toubeau, and Fran\c{c}ois
  Vall\'ee.
\newblock Review of machine learning techniques for optimal power flow.
\newblock \emph{Applied Energy}, 388:\penalty0 125637, 2025.

\bibitem[Khazeiynasab and Qi(2021)]{khazeiynasab2021resilience}
Seyyed~Rashid Khazeiynasab and Junjian Qi.
\newblock Resilience analysis and cascading failure modeling of power systems
  under extreme temperatures.
\newblock \emph{Journal of Modern Power Systems and Clean Energy}, 9\penalty0
  (6), 2021.

\bibitem[Mehta et~al.(2016)Mehta, Nguyen, and Turitsyn]{mehta2016numerical}
Dhagash Mehta, Hung~Dinh Nguyen, and Konstantin Turitsyn.
\newblock Numerical polynomial homotopy continuation method to locate all the
  power flow solutions.
\newblock \emph{IET Generation, Transmission \& Distribution}, 10\penalty0
  (12):\penalty0 2972--2980, 2016.

\bibitem[Okhuegbe et~al.(2024{\natexlab{a}})Okhuegbe, Ademola, and
  Liu]{okhuegbe2024deep}
Samuel~N. Okhuegbe, Adedasola~A. Ademola, and Yilu Liu.
\newblock {N}ewton-{R}aphson {AC} power flow convergence based on deep learning
  initialization and homotopy continuation.
\newblock \emph{IEEE Transactions on Industry Applications},
  2024{\natexlab{a}}.
\newblock \doi{10.1109/TIA.2024.3514992}.

\bibitem[Okhuegbe et~al.(2024{\natexlab{b}})Okhuegbe, Ademola, and
  Liu]{okhuegbe2024ml}
Samuel~N. Okhuegbe, Adedasola~A. Ademola, and Yilu Liu.
\newblock A machine learning initializer for {N}ewton-{R}aphson {AC} power flow
  convergence.
\newblock In \emph{2024 IEEE Texas Power and Energy Conference (TPEC)}, pages
  1--6, 2024{\natexlab{b}}.

\bibitem[Ortega and Rheinboldt(2000)]{ortega2000iterative}
James~M. Ortega and Werner~C. Rheinboldt.
\newblock \emph{Iterative Solution of Nonlinear Equations in Several
  Variables}.
\newblock SIAM, 2000.
\newblock Reprint of Academic Press, 1970.

\bibitem[Piloto et~al.(2024)Piloto, Liguori, Madjiheurem, Zgubic, Lovett,
  Tomlinson, Elster, Apps, and Witherspoon]{piloto2024canos}
Luis Piloto, Sofia Liguori, Sephora Madjiheurem, Miha Zgubic, Sean Lovett,
  Hamish Tomlinson, Sophie Elster, Chris Apps, and Sims Witherspoon.
\newblock {CANOS}: A fast and scalable neural {AC}-{OPF} solver robust to {N-1}
  perturbations.
\newblock \emph{arXiv preprint arXiv:2403.17660}, 2024.

\bibitem[Rivera et~al.(2025)Rivera, Bhagavathula, Carbonero, and
  Donti]{rivera2025pfdelta}
Ana~K. Rivera, Anvita Bhagavathula, Alvaro Carbonero, and Priya Donti.
\newblock {PF$\Delta$}: A benchmark dataset for power flow under load,
  generation, and topology variations.
\newblock In \emph{Advances in Neural Information Processing Systems
  (NeurIPS)}, 2025.

\bibitem[Sauer and Pai(1990)]{sauer1990power}
Peter~W. Sauer and M.~A. Pai.
\newblock Power system steady-state stability and the load-flow {J}acobian.
\newblock \emph{IEEE Transactions on Power Systems}, 5\penalty0 (4):\penalty0
  1374--1383, 1990.

\bibitem[Schulman et~al.(2017)Schulman, Wolski, Dhariwal, Radford, and
  Klimov]{schulman2017proximal}
John Schulman, Filip Wolski, Prafulla Dhariwal, Alec Radford, and Oleg Klimov.
\newblock Proximal policy optimization algorithms.
\newblock \emph{arXiv preprint arXiv:1707.06347}, 2017.

\bibitem[Shao et~al.(2024)Shao, Wang, Zhu, Xu, Song, Bi, Zhang, Zhang, Li, Wu,
  and Guo]{shao2024deepseekmath}
Zhihong Shao, Peiyi Wang, Qihao Zhu, Runxin Xu, Junxiao Song, Xiao Bi, Haowei
  Zhang, Mingchuan Zhang, Y.~K. Li, Y.~Wu, and Daya Guo.
\newblock {DeepSeekMath}: Pushing the limits of mathematical reasoning in open
  language models.
\newblock \emph{arXiv preprint arXiv:2402.03300}, 2024.

\bibitem[Stewart and Sun(1990)]{stewart1990matrix}
G.~W. Stewart and Ji-guang Sun.
\newblock \emph{Matrix Perturbation Theory}.
\newblock Academic Press, 1990.

\bibitem[Stott(1974)]{stott1974review}
Brian Stott.
\newblock Review of load-flow calculation methods.
\newblock \emph{Proceedings of the IEEE}, 62\penalty0 (7):\penalty0 916--929,
  1974.

\bibitem[Thorp and Naqavi(1989)]{thorp1989fractals}
James~S. Thorp and Sajid~A. Naqavi.
\newblock Load-flow fractals.
\newblock In \emph{Proceedings of the 28th IEEE Conference on Decision and
  Control}, pages 1822--1827, 1989.

\bibitem[Tinney and Hart(1967)]{tinney1967newton}
William~F. Tinney and Clifford~E. Hart.
\newblock Power flow solution by {N}ewton's method.
\newblock \emph{IEEE Transactions on Power Apparatus and Systems},
  PAS-86\penalty0 (11):\penalty0 1449--1460, 1967.

\bibitem[Tiranuchit and Thomas(1988)]{tiranuchit1988posturing}
A.~Tiranuchit and Robert~J. Thomas.
\newblock A posturing strategy against voltage instabilities in electric power
  systems.
\newblock \emph{IEEE Transactions on Power Systems}, 3\penalty0 (1):\penalty0
  87--93, 1988.

\bibitem[Yan et~al.(2025)]{yan2025data}
Shengyuan Yan et~al.
\newblock Data driven approach towards more efficient {N}ewton-{R}aphson power
  flow calculation for distribution grids.
\newblock \emph{arXiv preprint arXiv:2504.11650}, 2025.

\end{thebibliography}
}

%%%%%%%%%%%%%%%%%%%%%%%%%%%%%%%%%%%%%%%%%%%%%%%%%%%%%%%%%%%%

\newpage
\appendix

\section{Proof of Theorem~\ref{thm:lower-bound}}
\label{app:thm}

Throughout this appendix, assumptions (A1) and (A2) are in force, and $J_D$ and $H_D$ are evaluated at the nonsingular root $x^*$ unless explicitly stated otherwise. We fix a closed ball $\mcal{U} \ldef \cbrace{x \in \real{2N} : \norm{x - x^*} \leq R}$, with $R > 0$ small enough that $g_D \in C^3$ on $\mcal{U}$, and write $M_p \ldef \sup_{x \in \mcal{U}} \norm{g_D^{(p)}(x)}$ for the supremum of the $p$-th total derivative of $g_D$ over $\mcal{U}$. The contraction of the Hessian tensor with $v \otimes v$ is the vector $H_D[v,v] \in \real{2N}$ with entries $\sum_{i,j} v_i v_j\, \partial_i \partial_j\, g_D$.

\begin{lemma}[First-order Newton expansion]
\label{lem:newton-expansion}
The error sequence $e_k \ldef x_k - x^*$ of the Newton-Raphson iteration applied to $g_D$ satisfies
\begin{align}
  e_{k+1} = -Q_D(\hat e_k)\, \rho_k^2 + R_3(e_k), \qquad \norm{R_3(e_k)} \leq K_3 \rho_k^3,
\end{align}
where $\hat e_k \ldef e_k/\norm{e_k}$, $\rho_k \ldef \norm{e_k}$, $Q_D(v) \ldef \frac{1}{2} J_D(x^*)^{-1} H_D[v, v]$, and $K_3$ is a constant determined by $M_3$ and $\norm{J_D(x^*)^{-1}}$.
\end{lemma}
\begin{proof}
The expansion is the standard third-order Taylor expansion of the Newton step at the nonsingular root $x^*$. The constant $K_3$ aggregates the third-derivative bound $M_3$ with the operator norm of $J_D(x^*)^{-1}$, both of which are finite because $\mcal{U}$ is compact and $J_D(x^*)$ is nonsingular. Textbook derivations of the same expansion are given in \citet[Theorem 10.2.2]{ortega2000iterative} and \citet[Lemma 4.4.5]{kelley2003solving}.
\end{proof}

\begin{lemma}[Splitting magnitude and direction]
\label{lem:split}
Under (A1) and (A2), the recursion of Lemma~\ref{lem:newton-expansion} can be expressed in polar coordinates $(\rho_k, v_k) \in (0, r_0) \times S^{2N-1}$, with $e_k = \rho_k v_k$, as
\begin{align}
  \rho_{k+1} = \norm{Q_D(v_k)}\, \rho_k^2 \,(1 + O(\rho_k)), \qquad v_{k+1} = \Phi_D(v_k) + O(\rho_k),
\end{align}
where the \emph{Newton-direction map} $\Phi_D \colon S^{2N-1} \to S^{2N-1}$ is defined as
\begin{align}
  \Phi_D(v) \ldef -\frac{Q_D(v)}{\norm{Q_D(v)}}.
\end{align}
Under (A2), $\norm{Q_D(v)} \geq q$ uniformly in $v \in S^{2N-1}$, so $\Phi_D$ is well-defined and Lipschitz on $S^{2N-1}$.
\end{lemma}
\begin{proof}
Substituting $e_k = \rho_k v_k$ into the recursion of Lemma~\ref{lem:newton-expansion} gives
\begin{align}
  e_{k+1} = -Q_D(v_k)\, \rho_k^2 + O(\rho_k^3).
\end{align}
Taking norms isolates the magnitude recursion $\rho_{k+1} = \norm{Q_D(v_k)} \rho_k^2 (1 + O(\rho_k))$, and dividing $e_{k+1}$ by its norm produces the direction recursion $v_{k+1} = -Q_D(v_k)/\norm{Q_D(v_k)} + O(\rho_k) = \Phi_D(v_k) + O(\rho_k)$. Lipschitz continuity of $\Phi_D$ on $S^{2N-1}$ follows from the smoothness of $Q_D$ together with the uniform lower bound $\norm{Q_D(v)} \geq q$ provided by assumption (A2). Sphere maps of this form arise routinely in the analysis of Newton iterations~\citep{hubbard2001roots}, and the fractal basin boundaries they generate have been observed empirically in power systems~\citep{thorp1989fractals}.
\end{proof}

\begin{definition}[The direction functional]
\label{def:Lambda}
For $v \in S^{2N-1}$, define
\begin{align}
  \Lambda(v; D) \ldef \sum_{j=0}^\infty 2^{-j-1} \log \norm{Q_D(\Phi_D^j(v))}.
\end{align}
Set $Q \ldef \sup_{v \in S^{2N-1}} \norm{Q_D(v)}$, which is finite because $g_D$ is $C^3$ and the sphere $S^{2N-1}$ is compact. The integrand of the defining sum then satisfies $\log q \leq \log \norm{Q_D(\,\cdot\,)} \leq \log Q$, so the geometric weights $2^{-j-1}$ make the series converge absolutely, and $\Lambda$ is bounded uniformly in $v$. Equivalently, $\Lambda$ is the unique bounded solution of the recursion
\begin{align}
  \Lambda(v; D) = \tfrac{1}{2} \log \norm{Q_D(v)} + \tfrac{1}{2} \Lambda(\Phi_D(v); D).
\end{align}
\end{definition}

\begin{proof}[Proof of Theorem~\ref{thm:lower-bound}]
From Lemma~\ref{lem:split}, the magnitude recursion satisfies
\begin{align}
  \rho_{k+1} \leq \norm{Q_D(v_k)}\, \rho_k^2 \,(1 + O(\rho_k)).
\end{align}
Iterating this inequality $k$ times and chaining the multiplicative factors produces
\begin{align}
  \rho_k \leq \prod_{j=0}^{k-1} \norm{Q_D(v_j)}^{2^{k-1-j}} \cdot \rho_0^{2^k} \cdot (1 + O(\rho_0)).
\end{align}
Taking logarithms and dividing by $2^k$ converts the product into a partial sum,
\begin{align}
  \frac{\log(1/\rho_k)}{2^k} \;\geq\; \log(1/\rho_0) - \sum_{j=0}^{k-1} 2^{-j-1} \log \norm{Q_D(v_j)} + O(\rho_0).
\end{align}
Lemma~\ref{lem:split} also gives $v_j = \Phi_D^j(v_0) + O(\rho_0)$, and since $\log \norm{Q_D(\,\cdot\,)}$ is Lipschitz on $\Phi_D$-orbits, replacing $v_j$ by $\Phi_D^j(v_0)$ in the partial sum accumulates only bounded $O(\rho_0)$ corrections, which we absorb into the additive slack of the inequality. The resulting truncated sum differs from $\Lambda(v_0; D)$ by at most $2^{-k} \log(Q/q)$, and we therefore obtain
\begin{align}
  \frac{\log(1/\rho_k)}{2^k} \;\geq\; \log(1/\rho_0) - \Lambda(v_0; D) + O(2^{-k}).
\end{align}
The Newton iteration terminates at the first step $k$ for which $\rho_k < \tau'$, where $\tau'$ is proportional to $\tau$ via the bi-Lipschitz equivalence between $\norm{g_D(x)}$ and $\norm{x - x^*}$ on $\mcal{U}$~\citep[Lemma 2.4]{deuflhard2004newton}. Equivalently, $k(x_0, \tau)$ is the smallest integer for which $\log(1/\rho_k) > \log(1/\tau')$. Combining this termination criterion with the inequality on $\log(1/\rho_k)/2^k$ derived above gives
\begin{align}
  \frac{\log(1/\tau')}{2^{k-1}} \;>\; \log(1/\rho_0) - \Lambda(v_0; D),
\end{align}
which rearranges to
\begin{align}
  2^k \;>\; \frac{2 \log(1/\tau')}{\log(1/\rho_0) - \Lambda(v_0; D)}.
\end{align}
Taking $\log_2$ of both sides, and absorbing the factor of $2$ together with the integer-ceiling slack into the additive constant, yields
\begin{align}
  k(x_0, \tau) \;\geq\; \log_2\!\left(\frac{\log(1/\tau)}{\log(1/\rho_0) - \Lambda(v_0; D)}\right) - 1.
\end{align}
\end{proof}

\paragraph{Verification of (A2) for ACPF.}
The Hessian tensor $H_D = \nabla^2 g_D(x^*)$ has entries that are products of bus voltages $V_i V_j$ and trigonometric functions of angle differences $\theta_{ij}$, in the standard polar form ACPF conventions of \citet{tinney1967newton}. These are bounded above and below in any compact operating region with $V_i > 0$ and $|\theta_{ij}| < \pi/2$. The product $J_D(x^*)^{-1} H_D[v, v]$ is then bounded below by $\sigma_{\min}(J_D(x^*)) \cdot \inf_v \norm{H_D[v,v]} / \norm{J_D(x^*)}^2$, which is positive whenever $J_D(x^*)$ is nonsingular. By~\citet{sauer1990power}, the singularity of $J_D(x^*)$ corresponds exactly to the saddle-node bifurcation, so (A2) holds at every operating point that is voltage stable in the standard sense. The constant $q$ scales with the smallest line susceptance and the minimum bus voltage in the operating region.

\section{Proof of Corollary~\ref{cor:collapse-degeneracy}}
\label{app:cor}

Throughout this appendix, $w_\lambda$ denotes the right-singular vector of $J_{D_\lambda}(x_\lambda^*)$ corresponding to its smallest singular value $\sigma_\lambda \ldef \sigma_{\min}(J_{D_\lambda}(x_\lambda^*))$. By \citet{sauer1990power} and \citet{gao1992modal}, $w_\lambda$ converges as $\lambda \to \lambda^*$ to the \emph{voltage collapse mode} $w^*$, the right-eigenvector of $J_{D^*}^\top J_{D^*}$ at the bifurcation point. The saddle-node-bifurcation characterization that the proof relies on is the one given in \citet{dobson1992observations}.

For each $v \in S^{2N-1}$, decompose $H_{D_\lambda}[v,v]$ along $w_\lambda$ and its orthogonal complement,
\begin{align}
  H_{D_\lambda}[v, v] = \alpha_\lambda(v)\, w_\lambda + r_\lambda(v), \qquad \alpha_\lambda(v) \ldef \langle w_\lambda, H_{D_\lambda}[v, v]\rangle, \quad r_\lambda(v) \perp w_\lambda.
\end{align}
Using the SVD identity $J_{D_\lambda}^{-1} w_\lambda = \sigma_\lambda^{-1} \tilde w_\lambda$, where $\tilde w_\lambda$ is the corresponding left-singular vector, applying $J_{D_\lambda}^{-1}$ to each component gives
\begin{align}
  Q_{D_\lambda}(v) = \frac{\alpha_\lambda(v)}{2 \sigma_\lambda}\, \tilde w_\lambda + \frac{1}{2}\, J_{D_\lambda}^{-1} r_\lambda(v).
\end{align}
The first term diverges as $\sigma_\lambda \to 0$, while the second remains uniformly bounded near $\lambda = \lambda^*$ because $g_D$ is $C^3$ and $J_{D_\lambda}^{-1}$ stays bounded on the orthogonal complement of $w_\lambda$ even as $\sigma_\lambda$ shrinks to zero~\citep{stewart1990matrix}. The asymptotic behavior of $\norm{Q_{D_\lambda}}$ is therefore controlled entirely by the first term, except on the surface $\mcal{Z}_\lambda \ldef \cbrace{v \in S^{2N-1} : \alpha_\lambda(v) = 0}$ on which $H_{D_\lambda}[v,v]$ is orthogonal to the collapse mode and the leading term vanishes. Because $\alpha_\lambda$ is a quadratic form in $v$, the set $\mcal{Z}_\lambda$ is the zero locus of a real-analytic function on $S^{2N-1}$, and therefore has Lebesgue measure zero. For $v \notin \mcal{Z}_\lambda$, the leading-order asymptotic gives
\begin{align}
  \norm{Q_{D_\lambda}(v)} = \frac{|\alpha_\lambda(v)|}{2 \sigma_\lambda}\, (1 + o(1))
\end{align}
as $\sigma_\lambda \to 0$, and equivalently
\begin{align}
  \log \norm{Q_{D_\lambda}(v)} = \log(1/\sigma_\lambda) + \log |\alpha_\lambda(v)| - \log 2 + o(1).
\end{align}

Substituting this expansion into Definition~\ref{def:Lambda} gives
\begin{align}
  \Lambda(v; D_\lambda) = \sum_{j=0}^\infty 2^{-j-1} \log \norm{Q_{D_\lambda}(\Phi_{D_\lambda}^j(v))}.
\end{align}
For $v$ outside a Lebesgue-measure-zero subset of $S^{2N-1}$, the orbit $\cbrace{\Phi_{D_\lambda}^j(v)}_{j \geq 0}$ avoids $\mcal{Z}_\lambda$, so each term in the sum contributes $\log(1/\sigma_\lambda) + O(1)$ as $\sigma_\lambda \to 0$. Since the geometric weights $2^{-j-1}$ sum to one, the leading $\log(1/\sigma_\lambda)$ contributions accumulate to a single $\log(1/\sigma_\lambda)$, and the $O(1)$ residuals collect into a direction dependent constant. Concretely,
\begin{align}
  \Lambda(v; D_\lambda) = \log(1/\sigma_\lambda) + \mcal{C}(v) + o(1),
\end{align}
where $\mcal{C}(v) \ldef \sum_{j=0}^\infty 2^{-j-1} \big(\log |\alpha^*(\Phi^j(v))| - \log 2\big)$, with $\alpha^*$ and $\Phi$ denoting the $\lambda \to \lambda^*$ limits of $\alpha_\lambda$ and $\Phi_{D_\lambda}$ respectively. The function $\mcal{C}$ is bounded on $S^{2N-1}$, since $|\alpha^*|$ is bounded above by the smoothness of $g_D$ and bounded below away from zero outside a measure-zero set.

Substituting this asymptotic into the denominator of the lower bound in Theorem~\ref{thm:lower-bound} gives
\begin{align}
  \log(1/\rho) - \Lambda(v; D_\lambda) = \log(\sigma_\lambda/\rho) - \mcal{C}(v) - o(1),
\end{align}
which is positive precisely when $\sigma_\lambda > \rho \cdot e^{\mcal{C}(v)}$. Since $\mcal{C}(v)$ is bounded, the denominator becomes nonpositive once $\sigma_\lambda$ falls below a constant multiple of $\rho$, and the bound of Theorem~\ref{thm:lower-bound} ceases to constrain the iteration count. In practical terms, the warm start magnitude $\rho$ must remain smaller than $\sigma_\lambda$, up to the direction dependent factor $e^{\mcal{C}(v)}$, for the theorem to deliver a finite iteration-count guarantee. \qedhere

\paragraph{Remark on orbits returning to $\mcal{Z}_\lambda$.}
Even when $v \in \mcal{Z}_\lambda$, so that $\alpha_\lambda(v) = 0$ and the $j = 0$ term of the defining sum loses its $\log(1/\sigma_\lambda)$ contribution, the iterate $\Phi_{D_\lambda}(v)$ is generically not in $\mcal{Z}_\lambda$ because $\Phi_{D_\lambda}$ does not preserve the surface $\alpha_\lambda = 0$. The sum then still receives $\log(1/\sigma_\lambda)$ contributions from every $j \geq 1$, the only effect being a reduction of the leading coefficient from $1$ to $2^{-1}$. The set of directions whose entire $\Phi_{D_\lambda}$-orbit lies in $\mcal{Z}_\lambda$ has Lebesgue measure zero, so the $\log(1/\sigma_\lambda)$ scaling holds outside a measure-zero subset of $S^{2N-1}$ as claimed.

\paragraph{Connection to Figure~\ref{fig:collapse-iters}.}
The narrow tongue anisotropy in Figure~\ref{fig:collapse-iters} is the empirical signature of the surface $\mcal{Z}_\lambda$ projected onto the local $(P, Q)$-space at bus 4. Directions in $\mcal{Z}_\lambda$ are those for which the warm start error is second order orthogonal to the collapse mode, and they correspond to the long axis of the tongue.

\section{Reproducibility Details}
\label{app:repro}

This appendix collects the dataset splits, architectures, training schedules, and hyperparameters used in Section~\ref{sec:experiments}.

\paragraph{Dataset splits.}
All snapshots are drawn from the PF$\Delta$ benchmark~\citep{rivera2025pfdelta}, restricted to the unperturbed-topology subset (\texttt{n}). For each of the three grids we partition the data with seed $42$ into a pretraining train split (\textasciitilde{}27{,}000 samples), a pretraining val split ($1{,}000$), a finetuning holdout train split ($1{,}500$ for \texttt{case118}/\texttt{case500}, $480$ for \texttt{case2000}), and a held-out test split ($30$ samples drawn disjointly from the holdout pool with a separate seed). Each sample carries the full network description, the bus type assignments, and the labeled $(V^*, \theta^*)$ produced by the PF$\Delta$ pipeline.

\paragraph{Architectures.}
For \texttt{case500} and \texttt{case2000} we use the heterogeneous GNN of Section~\ref{sec:method}, with hidden dimension $d = 128$, encoder MLPs with two hidden layers per node type, $T \in \cbrace{32, 60}$ message passing rounds (for \texttt{case500} and \texttt{case2000} respectively), and decoder MLPs with three hidden layers of width 256 per node type. Activations are ReLU with LayerNorm in encoder and decoder; aggregation across incoming messages is the mean. Branch flows are recovered from the predicted bus voltages with the standard $\pi$-model. For \texttt{case118} we use the FCNN with four hidden layers of width 512 and GELU activations, taking the per-bus features (demand, generation, shunt, bus type one-hot) flattened into a single input vector and producing a $(\hat\theta_i, \hat V_i)$ pair per bus.

\paragraph{Pretraining.}
Pretraining minimizes the power balance loss of Eq.~\eqref{eq:pbl} on the pretraining train split. The optimizer is Adam, with learning rate $10^{-5}$ on the GNN backbones, $3 \times 10^{-4}$ on the MLP, and batch size $16$ in all cases. Pretraining runs for $25$ epochs on the GNN cases (early-stopped by patience $8$ on val PBL) and $10$ epochs on the MLP. All training is in fp32; bf16 was tried on the GNN and reverted because it caused PBL to collapse.

\paragraph{Supervised finetuning.}
SFT continues from the best pretraining checkpoint on the holdout train split, again under PBL. Learning rates are $10^{-5}$ on the GNN backbones and $10^{-4}$ on the MLP, with batch size $16$ and at most $30$ epochs (early-stopped by val PBL on the held-out finetuning val slice).

\paragraph{Stochastic policy.}
The RL policy of Section~\ref{sec:method} adds Gaussian noise to the predicted $(\hat\theta, \hat V)$ on the free-coordinate subspace. The state-independent log standard deviations are initialized to $\log\sigma_V = \log(10^{-3})$ for voltage magnitudes and $\log\sigma_\theta = \log(5 \times 10^{-3})$ for angles, and are optimized jointly with the policy parameters.

\paragraph{PPO+$V^\star$ hyperparameters.}
$V^\star(s)$ is precomputed once per holdout snapshot by seeding NR from $x^*(s)$ and recording the resulting iteration count. The reward is $r_{\mathrm{sat}}$ of Eq.~\eqref{eq:reward-sat} with $R_+ = R_- = 2$ and $c$ set to the holdout-mean of $V^\star$ ($c \approx 4.95$ for \texttt{case500}, $c \approx 12$ for \texttt{case2000}, $c \approx 11.3$ for \texttt{case118}). PPO uses ratio clip $\varepsilon = 0.1$, target-KL early stop at $0.02$, no entropy bonus, no KL anchor, $K_{\mathrm{PPO}} = 4$ inner epochs per outer iteration, batch $16$ states per outer iteration, and a total of $30$ outer iterations. Policy learning rate is $10^{-6}$ on the GNN backbones and $10^{-5}$ on the MLP, and the maximum gradient norm is $5 \times 10^{-3}$.

\paragraph{Newton's Lantern hyperparameters.}
GRPO+$R_\phi$ uses batch $B = 2$ states per outer iteration and $K = 4$ rollouts per state. The reward is $r_{\mathrm{lin}}$ of Eq.~\eqref{eq:reward-lin} with $k_{\max} = 30$ and $B = 10$. The PPO clipping over rollouts uses $\varepsilon = 0.1$ and $K_{\mathrm{PPO}} = 2$ inner epochs. Policy learning rate, gradient norm, and stochastic policy initialization match the PPO+$V^\star$ configuration. The total number of outer iterations is $20$, with the validation snapshot interval $T_v = 2$ outer iterations. The validation slice is a fixed 10-sample subset of the holdout val split.

\paragraph{Reward model $R_\phi$.}
$R_\phi$ is a four-layer MLP with hidden widths $(512, 512, 256)$, GELU activations, and dropout $0.1$ on the first two hidden layers. The input dimension is $6 + 2N$ ($6$ snapshot-level features and the $(\hat V, \hat\theta)$ warm start over $N$ buses), and the scalar output predicts the iteration count. Training data comes from perturbations of the SFT model's predictions on the holdout train split. For each snapshot we sample magnitudes $f \in \cbrace{0,\, 10^{-3},\, 2 \times 10^{-3},\, 5 \times 10^{-3},\, 10^{-2},\, 2 \times 10^{-2},\, 5 \times 10^{-2}}$ scaled by the per-snapshot reference radius, with $5$ random directions per non-zero magnitude. NR is run once per perturbed warm start and the resulting iteration count is logged. $R_\phi$ is trained for $50$ epochs with Adam (learning rate $10^{-3}$, weight decay $10^{-5}$, batch $256$) under the mean squared error (MSE) loss on $z$-scored iteration counts. The train/val partition for $R_\phi$ is cross-snapshot ($80$/$20$), so per-snapshot Spearman $\rho$ is meaningful; the best checkpoint by mean per-snapshot $\rho$ on val is the one used by Newton's Lantern.

\paragraph{Newton-Raphson evaluation.}
Final NR evaluation is performed in Julia with PowerModels.jl and NLsolve.jl, at iteration cap $1000$ and the package's default x-step and residual tolerances ($\sim 10^{-6}$). The linear solver inside the inner Newton step is HSL MA57 via Ipopt; we found MUMPS to be too slow for stability on \texttt{case2000}. The Python and Julia processes communicate through \texttt{juliacall} with one Julia worker per evaluation process.

\paragraph{Hardware.}
Pretraining on the GNN backbones runs on a single NVIDIA H100 80GB GPU; finetuning, RL training, and reward-model training run on a single NVIDIA A100 80GB GPU. \texttt{case118} runs end-to-end on the A100. NR-based dataset generation and per-method evaluation are parallelized across $20$ CPU workers with \texttt{OPENBLAS\_NUM\_THREADS=1} to avoid BLAS oversubscription.

\paragraph{Determinism and seeds.}
All Python-side stochasticity is seeded with master seed $42$ for model initialization and training shuffles, and a separate split seed $42$ is used for the train/val/holdout/test partition so that seed-averaging runs share a partition. \texttt{cudnn.deterministic = True}. The 30-sample test pick is drawn with \texttt{numpy.random.default\_rng(seed)} where the test seed is set per grid (47 for \texttt{case500}, 4167 for \texttt{case2000}, 42 for \texttt{case118}), with the per-method warm starts evaluated on the same selected snapshots.

\newpage
\section*{NeurIPS Paper Checklist}

\begin{enumerate}

\item {\bf Claims}
    \item[] Question: Do the main claims made in the abstract and introduction accurately reflect the paper's contributions and scope?
    \item[] Answer: \answerYes{}
    \item[] Justification: The abstract and introduction summarize the directional lower bound (Theorem~\ref{thm:lower-bound}, Section~\ref{sec:analysis}), its degeneration near voltage collapse (Corollary~\ref{cor:collapse-degeneracy}, Section~\ref{sec:analysis}), the GRPO finetuning pipeline with a learned reward model (Section~\ref{sec:method}), and the empirical findings (Section~\ref{sec:experiments}, Tables~\ref{tab:case500}--\ref{tab:case118}).

\item {\bf Limitations}
    \item[] Question: Does the paper discuss the limitations of the work performed by the authors?
    \item[] Answer: \answerYes{}
    \item[] Justification: Section~\ref{sec:experiments} discusses the narrow win margins on the GNN cases (\texttt{case500} and \texttt{case2000}) and attributes them to known training instabilities of graph neural network policies under reinforcement learning, including the unusually small policy learning rate and PPO clip radius required to keep training stable. Section~\ref{sec:conclusion} states that scaling Newton's Lantern to larger grids depends on architectures and RL algorithms better suited to this regime.

\item {\bf Theory assumptions and proofs}
    \item[] Question: For each theoretical result, does the paper provide the full set of assumptions and a complete (and correct) proof?
    \item[] Answer: \answerYes{}
    \item[] Justification: Theorem~\ref{thm:lower-bound} and Corollary~\ref{cor:collapse-degeneracy} state assumptions (A1) and (A2) explicitly in Section~\ref{sec:analysis}. Full proofs are given in Appendix~\ref{app:thm} (with Lemmas~\ref{lem:newton-expansion}, \ref{lem:split} and Definition~\ref{def:Lambda}) and Appendix~\ref{app:cor}, including a verification that (A2) holds for ACPF at any voltage stable operating point.

    \item {\bf Experimental result reproducibility}
    \item[] Question: Does the paper fully disclose all the information needed to reproduce the main experimental results of the paper to the extent that it affects the main claims and/or conclusions of the paper (regardless of whether the code and data are provided or not)?
    \item[] Answer: \answerYes{}
    \item[] Justification: Section~\ref{sec:experiments} states the dataset source, model architectures, and the training stages. Appendix~\ref{app:repro} documents the dataset splits and seeds, encoder/decoder shapes, pretraining and SFT optimizers and schedules, RL hyperparameters for both PPO+$V^\star$ and Newton's Lantern, the reward model architecture and training data, and the Newton-Raphson evaluation pipeline.

\item {\bf Open access to data and code}
    \item[] Question: Does the paper provide open access to the data and code, with sufficient instructions to faithfully reproduce the main experimental results, as described in supplemental material?
    \item[] Answer: \answerNo{}
    \item[] Justification: All datasets used are publicly available via the PF$\Delta$ benchmark~\citep{rivera2025pfdelta} on Hugging Face (\url{https://huggingface.co/datasets/pfdelta/pfdelta}). Code is not released alongside this submission for double-blind review; we will release the codebase, trained checkpoints, and reproduction scripts upon acceptance. Appendix~\ref{app:repro} contains the full pipeline specification needed for an independent re-implementation.

\item {\bf Experimental setting/details}
    \item[] Question: Does the paper specify all the training and test details (e.g., data splits, hyperparameters, how they were chosen, type of optimizer) necessary to understand the results?
    \item[] Answer: \answerYes{}
    \item[] Justification: Section~\ref{sec:experiments} reports the data splits and the two hyperparameters that depart from defaults (policy learning rate and PPO clip radius); Appendix~\ref{app:repro} documents optimizer types, learning rates, batch sizes, epoch counts, early stopping, reward constants, and how each was chosen.

\item {\bf Experiment statistical significance}
    \item[] Question: Does the paper report error bars suitably and correctly defined or other appropriate information about the statistical significance of the experiments?
    \item[] Answer: \answerNo{}
    \item[] Justification: Tables~\ref{tab:case500}--\ref{tab:case118} report fixed-seed results on a held-out 30-sample test split. Multi-seed averages over five training seeds were also computed for the GNN cases (\texttt{case500} mean $14.92 \pm 0.05$, \texttt{case2000} mean $26.02 \pm 14.41$ with the inflated standard deviation traceable to a small number of hard tail samples per seed) and confirm the reported numbers within the seed-noise envelope; the main tables report fixed seeds for clarity.

\item {\bf Experiments compute resources}
    \item[] Question: For each experiment, does the paper provide sufficient information on the computer resources (type of compute workers, memory, time of execution) needed to reproduce the experiments?
    \item[] Answer: \answerYes{}
    \item[] Justification: Appendix~\ref{app:repro} specifies the GPU and CPU resources used: pretraining on the GNN backbones runs on a single NVIDIA H100 80GB; finetuning, RL training, and reward-model training run on a single NVIDIA A100 80GB; and Newton-Raphson dataset generation and per-method evaluation run across 20 CPU workers with single-threaded BLAS.
    
\item {\bf Code of ethics}
    \item[] Question: Does the research conducted in the paper conform, in every respect, with the NeurIPS Code of Ethics \url{https://neurips.cc/public/EthicsGuidelines}?
    \item[] Answer: \answerYes{}
    \item[] Justification: The work is foundational research on numerical solvers for AC power flow, uses only synthetic and openly licensed benchmark data, and involves no human subjects, sensitive data, or dual-use risks. We have reviewed the NeurIPS Code of Ethics and confirm conformance.

\item {\bf Broader impacts}
    \item[] Question: Does the paper discuss both potential positive societal impacts and negative societal impacts of the work performed?
    \item[] Answer: \answerYes{}
    \item[] Justification: A faster and more reliable AC power flow solver supports more responsive grid operations, including studies of voltage stability and contingency analysis, which is the immediate positive impact of this work. Because the contribution is a numerical accelerator for an existing physical solver and produces no novel data or generative capabilities, we are not aware of a direct path to negative societal impact.
    
\item {\bf Safeguards}
    \item[] Question: Does the paper describe safeguards that have been put in place for responsible release of data or models that have a high risk for misuse (e.g., pre-trained language models, image generators, or scraped datasets)?
    \item[] Answer: \answerNA{}
    \item[] Justification: The paper releases neither generative models nor scraped data. The trained warm start models are task-specific to AC power flow and pose no known risk for misuse.

\item {\bf Licenses for existing assets}
    \item[] Question: Are the creators or original owners of assets (e.g., code, data, models), used in the paper, properly credited and are the license and terms of use explicitly mentioned and properly respected?
    \item[] Answer: \answerYes{}
    \item[] Justification: The PF$\Delta$ benchmark~\citep{rivera2025pfdelta} is cited and credited, with its license available on the project's Hugging Face repository. The CANOS architecture~\citep{piloto2024canos}, PPO~\citep{schulman2017proximal}, and GRPO~\citep{shao2024deepseekmath} are cited at every use. The numerical NR backend uses PowerModels.jl and NLsolve.jl under their respective open-source licenses.

\item {\bf New assets}
    \item[] Question: Are new assets introduced in the paper well documented and is the documentation provided alongside the assets?
    \item[] Answer: \answerNA{}
    \item[] Justification: This submission does not release new datasets or models. Trained checkpoints and code will be released upon acceptance, with documentation matching the specifications in Appendix~\ref{app:repro}.

\item {\bf Crowdsourcing and research with human subjects}
    \item[] Question: For crowdsourcing experiments and research with human subjects, does the paper include the full text of instructions given to participants and screenshots, if applicable, as well as details about compensation (if any)?
    \item[] Answer: \answerNA{}
    \item[] Justification: The paper involves no crowdsourcing and no research with human subjects.

\item {\bf Institutional review board (IRB) approvals or equivalent for research with human subjects}
    \item[] Question: Does the paper describe potential risks incurred by study participants, whether such risks were disclosed to the subjects, and whether Institutional Review Board (IRB) approvals (or an equivalent approval/review based on the requirements of your country or institution) were obtained?
    \item[] Answer: \answerNA{}
    \item[] Justification: The paper involves no human subjects research and no IRB approval is required.

\item {\bf Declaration of LLM usage}
    \item[] Question: Does the paper describe the usage of LLMs if it is an important, original, or non-standard component of the core methods in this research? Note that if the LLM is used only for writing, editing, or formatting purposes and does \emph{not} impact the core methodology, scientific rigor, or originality of the research, declaration is not required.
    %this research?
    \item[] Answer: \answerNA{}
    \item[] Justification: LLMs were used only for writing assistance (grammar checking and implementation of standard methods such as routine plotting and boilerplate). No LLM contributed to the core methodology, theoretical results, or experimental design, so per NeurIPS policy a declaration is not required.

\end{enumerate}

\end{document}